\documentclass[12pt]{article}

\usepackage[natbibapa]{apacite}
\usepackage{natbib}
\usepackage[utf8]{inputenc} 
\usepackage[T1]{fontenc}    

\usepackage{lmodern} 

\usepackage{hyperref,url,booktabs,microtype}  
\usepackage{amsfonts,amssymb,amsmath,amsthm,amstext,amssymb,mathtools}
\usepackage{caption,subcaption,float,color,graphicx,enumerate}
\usepackage{dsfont,
					fullpage,
					parskip,
					paralist,
					nicefrac,
					listings,
					array}  
\newcolumntype{C}[1]{>{\centering\arraybackslash}m{#1}}

\usepackage{tikz}
\usetikzlibrary{bayesnet,er,trees,shapes.symbols,mindmap,arrows,decorations,shapes.misc,shapes.arrows,chains,matrix,positioning,scopes,decorations.pathmorphing,patterns}
\usepackage{tikz-cd} 

\pgfdeclarelayer{bg}    
\pgfsetlayers{bg,main}  

\usepackage{thmtools,thm-restate}

\newcommand{\blind}{0}

\newcommand{\trimmedplot}[1]{\includegraphics[width=0.49\linewidth,trim={5mm 1mm 6mm 4.8mm},clip]{#1}}
\newcommand*{\Cdot}{\setbox0\hbox{$x$}\hbox to\wd0{\hss$\cdot$\hss}}

\newcommand*{\bR}{\mathbb{R}}
\newcommand*{\bN}{\mathbb{N}}

\newcommand*{\bE}{\mathbb{E}}
\newcommand*{\bV}{\mathbb{V}}
\newcommand*{\bP}{\mathbb{P}}
\newcommand*{\cN}{\mathcal{N}}

\newcommand*{\Cov}{\mathbb{C}\mathrm{ov}}

\definecolor{darkgreen}{rgb}{0,0.4,0}

\makeatletter
\let\oldabs\abs
\def\abs{\@ifstar{\oldabs}{\oldabs*}}
\let\oldnorm\norm
\def\norm{\@ifstar{\oldnorm}{\oldnorm*}}
\makeatother

\newcommand{\FIXME}[1]{}

\def\cN{\mathcal N}
\def\mkx{K_X}
\def\mkz{K_Z}

\newcommand{\param}{\theta}

\newcommand{\targfunc}{f} 

\newcommand{\acc}{\alpha} 

\newcommand{\evid}{\mathcal{E}} 
\newcommand{\trm}{\mu} 
\newcommand{\trma}{\trm_Y} 
\newcommand{\trd}{\rho} 
\newcommand{\trda}{\trd_{Y}} 
\newcommand{\trdae}{\widehat{\trd}_Y} 
\newcommand{\targdim}{d}
\newcommand{\prdunparam}{q} 
\newcommand{\prd}{{\prdunparam}} 

\newcommand{\IS}{S^{\rm IS}} 
\newcommand{\ISe}{\hat S^{\rm IS}} 
\newcommand{\ISs}{\widetilde S^{\rm IS}} 
\newcommand{\nsmp}{K} 

\theoremstyle{plain}
\newtheorem{theorem}{Theorem}[section]
\newtheorem{proposition}[theorem]{Proposition}

\theoremstyle{definition}
\newtheorem{definition}[theorem]{Definition}

\newtheorem{algorithm}[theorem]{Algorithm}

\newtheorem{remark}[theorem]{Remark}

\newtheorem{assumption}[theorem]{Assumption}

\begin{document}

\def\spacingset#1{\renewcommand{\baselinestretch}%
{#1}\small\normalsize} \spacingset{1}


\if0\blind
{
  \title{\bf Markov Chain Importance Sampling --\\ a highly efficient estimator for MCMC}
  \author{
  	Ingmar Schuster\\
  	Zalando Research, Zalando SE, Berlin\\
  	and \\
  	Ilja Klebanov \\
  	Zuse Institute Berlin, Berlin}
  \maketitle
}
\newcommand{\unblindinfo}[2]{#1}

 \fi

\if1\blind
{
  \bigskip
  \bigskip
  \bigskip
  \begin{center}
    {\LARGE\bf Markov Chain Importance Sampling --\\ a highly efficient estimator for MCMC}
\end{center}
  \medskip
  
} 

 \newcommand{\unblindinfo}[2]{\emph{[#2 withheld to not unblind reviewers]}}

\fi

\bigskip
\begin{abstract}
Markov chain (MC) algorithms are ubiquitous in machine learning and statistics and many other disciplines. 
Typically, these algorithms can be formulated as acceptance rejection methods.
In this work we present a novel estimator applicable to these methods, dubbed Markov chain importance sampling (MCIS), which efficiently makes use of rejected proposals.
For the unadjusted Langevin algorithm, it provides a novel way of correcting the discretization error.
Our estimator satisfies a central limit theorem and improves on error per CPU cycle, often to a large extent.
As a by-product it enables estimating the normalizing constant, an important quantity in Bayesian machine learning and statistics.
\end{abstract}

\noindent%

\spacingset{1.45}

\section{Introduction}
Markov chain Monte Carlo (MCMC) methods, such as 
Metropolis--Hastings \citep[MH;][]{hastings1970monte} or unadjusted Langevin  \citep[ULA;][]{parisi1981correlation,roberts1996exponential} algorithms, constitute a widely used tool to compute expected values $\bE_\trm[\targfunc] = \int_{\bR^d}\targfunc(x)\, \mathrm d\trm(x)$
of a function $\targfunc\in L^1(\trm)$ with respect to a probability measure $\trm$.
MCMC methods can tackle high dimensions as well as scenarios where the probability density of $\trm$ is only accessible up to an unknown normalizing constant $\evid>0$, i.e.\ we are given a density $\trd$ satisfying
\[
\mathrm d\trm (x) = \frac{\trd(x)}{\evid}\, \mathrm dx,
\qquad
\evid = \int_{\bR^d} \rho(x)\, \mathrm dx,
\]
a situation often occurring in Bayesian inverse problems.
MCMC algorithms can be formulated as acceptance rejection methods (not to be confused with accept-reject samplers, compare Chapter 2.3 of \citealt{christian2007monte}):
\begin{algorithm}[generic Markov chain acceptance rejection algorithm]
	\label{alg:AcceptReject}
	\ \\
	Let $\trd$ be a probability density on $\bR^d$ called the \emph{target density}, $Q = (\prd(\Cdot |x))_{x\in\bR^d}$ be a family of probability densities on $\bR^d$ referred to as \emph{proposal densities} and $\acc\colon \bR^{d}\times \bR^{d} \to [0,1]$ be a function, which may depend on $\trd$ and $Q$, called \emph{acceptance probability function}.
	Starting with some initial point $X_1\in\bR^d$, iterate for $k\in\bN$:
	\begin{compactenum}[(i)]
		\item Draw a sample $Y_k\sim\prd(\Cdot|X_k)$, independent of $X_1,\dots,X_{k-1}$ and $Y_1,\dots,Y_{k-1}$.
		\item Compute the acceptance probability $\acc_k = \acc(X_k,Y_k) \in [0,1]$.
		\item Set $X_{k+1} = Y_k$ with probability $\alpha_k$ (``accept'') and $X_{k+1} =X_k$ otherwise (``reject'').
	\end{compactenum}
\end{algorithm}
If the chain evolves without an explicit acceptance step, for example in the unadjusted Langevin algorithm (ULA, sometimes called discretized Langevin algorithm, see \citealt{parisi1981correlation,roberts1996exponential}), it can be rewritten in the form of Algorithm \ref{alg:AcceptReject} by setting $\acc(x,y) = 1$.

The most prominent class of acceptance rejection methods are MH algorithms, where the choice of the acceptance probability function
\begin{equation}
\label{equ:MHaccaptance}
\acc(x,y) = \min\left\{1, \frac{\prd(x|y)\trd(y)}{\prd(y|x)\trd(x)} \right\},
\qquad
x,y\in\bR^{d},
\end{equation}
guarantees $\trm$-reversibility and thereby $\trm$-stationarity of the Markov chain $(X_k)_{k\in\bN}$.
Numerous choices of the proposal density $\prd$ have been suggested in the literature, corresponding to a wide range of MH algorithms such as independent MH, random walk MH, Metropolis adjusted Langevin algorithm (MALA), Hamiltonian Monte Carlo (HMC), Gibbs sampling or the preconditioned Crank--Nicolson algorithm (pCN).

While most other acceptance rejection methods, including all MH algorithms, are designed to sample from the target distribution $\mu$, ULA provides a notable exception by sampling from an approximation $\trm_{X}$ of $\trm$ \citep{durmus2017nonasymptotic}.
We will make the following general assumption concerning the stationary distribution of the Markov chain $(X_k)$:
\begin{assumption}
\label{assumption:MCadmitsUniqueStationaryDensity}
The Markov chain $(X_k)$ in Algorithm \ref{alg:AcceptReject} admits a unique stationary distribution $\trm_{X}$ with probability density $\trd_X$.
\end{assumption}

While the computation of $\alpha_k$ usually depends on the evaluation $\trd(Y_k)$, which typically is the most time-consuming step (e.g.\ for Bayesian posteriors with a costly forward problem, such as solving a high-dimensional partial differential equation), this information is discarded whenever the proposed point $Y_k$ is rejected, resulting in a waste of resources.
This paper suggests a novel approach to remedy this drawback by providing new estimators for $\bE_\trm[\targfunc]$ based on all proposed samples $Y_k$, reweighted in a particular way using self-normalized importance sampling \citep[IS;][Chapter 9]{mcbook}.
This requires the marginal distribution of the samples $Y_k$, which is derived as a key step of our method and has not been studied in previous work to the best of our knowledge.
Note that this methodology does not modify the Markov chain itself in any way (apart from saving the rejected proposals), only the integral estimator is adjusted.

This paper is structured as follows. Section~\ref{sec:MCIS} introduces the Markov chain importance sampling estimator. Its convergence properties are studied in Section~\ref{section:ConvergenceProperties} and related work is discussed in Section~\ref{section:RelatedWork}. Numerical experiments are provided in Section~\ref{section:Numerical} followed by a conclusion.

\section{The Markov chain importance sampling estimator}
\label{sec:MCIS}
The application of IS requires an estimate for the asymptotic distribution $\trma$ of the proposed points $Y_k$, the probability density of which will be denoted by $\trda$.
The approximation $\trdae$ of $\trda$ will be computed using yet another Monte Carlo sum based on the samples $X_k$,
\begin{equation}
\label{equ:HowToEstimateTrda}
\trdae(y)
=
\frac{1}{K}\sum_{k=1}^{K} \prd(y|X_k)
\xrightarrow{K\to\infty}
\trda(y)
=
\int \trd_X(x) \prd(y|x)\, \mathrm dx,
\qquad
y\in\bR^{d}.
\end{equation}
Note that this step does not require any further evaluations of the target density and that $\rho_{Y}$ is a probability density function by the Fubini--Tonelli theorem.

The exact and approximate probability densities $\trda$ and $\trdae$ can now be used to formulate the following self-normalized importance sampling estimates:

\begin{definition}
	Let the chain $(X_k,Y_k)$ be generated by a Markov chain acceptance rejection algorithm (Algorithm \ref{alg:AcceptReject}) and let Assumption \ref{assumption:MCadmitsUniqueStationaryDensity} hold.
	We define the \emph{Markov chain importance sampling} (MCIS) estimators for $\bE_\trm[\targfunc]$ as
	\begin{align}
	\label{equ:generalIS}
	\IS_K(\targfunc)
	\coloneqq
	&\frac{\sum_{k=1}^{K} w(Y_k) \targfunc(Y_k)}{\sum_{k=1}^{K} w(Y_k)},
	\qquad
	w \coloneqq \frac{\trd}{\trda},
	\\
	\label{equ:generalISestimate}
	\ISe_K(\targfunc)
	\coloneqq
	&\frac{\sum_{k=1}^{K} \hat w(Y_k) \targfunc(Y_k)}{\sum_{k=1}^{K} \hat w(Y_k)},
	\qquad
	\hat w \coloneqq \frac{\trd}{\trdae},
	\end{align}
	where $\trda$ and $\trdae$ are given by \eqref{equ:HowToEstimateTrda}.
\end{definition}

While this methodology is of interest to most MH algorithms due to the ability to recycle rejected samples, it is particularly attractive for the unadjusted Langevin algorithm (ULA) by providing consistent estimators.
As mentioned in the introduction, ULA can be rewritten as a Markov chain acceptance rejection method (Algorithm \ref{alg:AcceptReject}) that accepts the proposals with probability $\acc(x,y)\equiv 1$.
Its proposal density
\[
\prd(\Cdot|x) = \mathcal{N}(\, \Cdot \,;\, x + \param \nabla \log \trd(y), 2\param I_d)
\]
results from the Euler-Maruyama discretization of the Langevin stochastic differential equation
\begin{equation}
\label{equ:OverdampedLangevinSDE}
\mathrm d V_t = \nabla \log \trd(V_t)\, \mathrm d t + \sqrt{2}\, \mathrm d B_t,
\end{equation}
the stationary density $\rho_V$ of which coincides with the target density $\evid^{-1}\trd$.
Here, $\param>0$ is a ``small'' time step parameter, $I_d$ denotes the $d$-dimensional identity matrix and $(B_t)$ is the Wiener process (also known as Brownian motion).
The discretization results in a disagreement between the asymptotic density $\trd_X$ of the resulting Markov chain\footnote{The existence of and convergence to $\trd_X$ can be guaranteed under rather mild assumptions using standard results for Markov chains, see e.g.\ \citet[Theorem 13.0.1]{meyn2012markov}.} $(X_k)$ and the target density $\evid^{-1}\trd$.
This disagreement introduces an asymptotic bias of ULA when using the vanilla MCMC estimator, while importance sampling provides consistent estimators (for independent and identically distributed (i.i.d.) samples see \citet[Theorem 9.2]{mcbook}, for the setup presented here, see the law of large numbers in Section \ref{section:LLN_CLT_MCIS}).

\begin{remark}
Note that our approach is practically applicable only if the proposal density $\prd(\Cdot|x)$ can be efficiently evaluated for any $x\in\bR^d$, which seemingly rules out certain algorithms such as HMC.
In the case of HMC this issue can be resolved by considering proposal densities in the phase space $\bR^{2d}$ rather than $\bR^{d}$, resulting in the so-called \emph{Shadow HMC} method \citep{izaguirre2004shadow} as a special case of MCIS, see the discussion in Section \ref{section:RelatedWork}.
\end{remark}

\begin{remark}
	\label{remark:Marginal proposal density form}
To get an intuition for $\trda$ in equation \eqref{equ:HowToEstimateTrda}, consider the standard case of Gaussian random walk proposals centered at the current Markov chain state.
In this case, $\trda$ is simply the convolution of $\evid^{-1}\trd$ with the Gaussian kernel, see Figure~\ref{fig:Marginal_proposal_density_comic} for an illustration.
 In fact, in many cases, $\trda$ will be a convolution of $\evid^{-1}\trd$ with some kernel and thereby have heavier tails than $\trd$ under mild conditions on $\trd$.
For the independent MH algorithm, where the proposal distribution does not depend on the current state, $\prd(\Cdot|x) = \prd$ for all $x\in\bR^d$, we obviously have $\trda = \prd$ \citep{tierney1994markov} and our algorithm is equivalent to standard importance sampling with proposal density $\prd$.
In most cases of interest, however, $\trda$ will not be available in closed form and we will have to rely on the Monte Carlo approximation given in \eqref{equ:HowToEstimateTrda}.
\end{remark}
In summary, the MCIS estimate \eqref{equ:generalISestimate} is applicable to all Markov chain acceptance rejection methods for which the proposal density $\prd$ can be evaluated efficiently.
Due to the use of importance sampling, it is consistent even if the asymptotic density $\trd_X$ of the Markov chain $(X_k)$ does not agree with the target density $\trd$ (the speed of convergence will, however, strongly depend on $\trd_X$).
It thereby provides a remedy for discarding the rejected samples in most MH algorithms (sometimes called \emph{waste recycling}) and compensates for the asymptotic bias of ULA.
\begin{figure}
	\centering
	\includegraphics[width=0.45\textwidth,trim={1.5mm 5mm 5mm 2.5mm},clip]{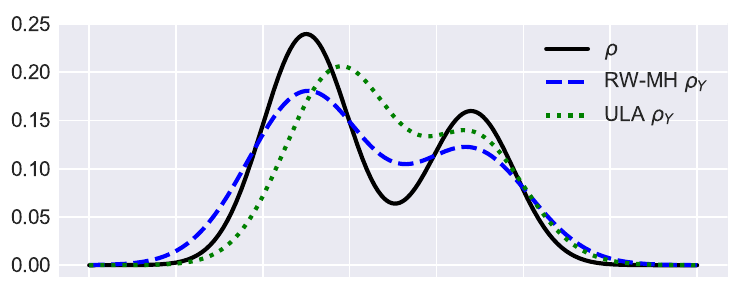}
	\hspace{8mm}
	\includegraphics[width=0.45\textwidth,trim={1.5mm 5mm 5mm 2.5mm},clip]{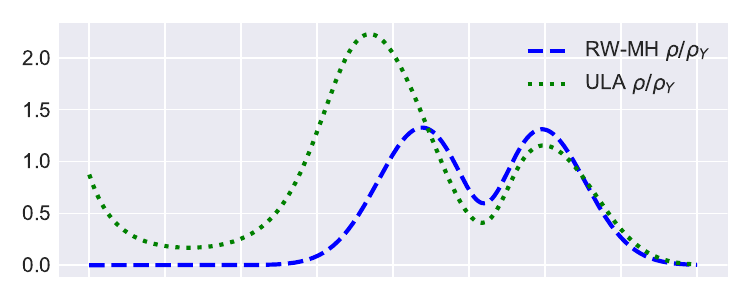}
	\caption{\emph{Left:} Marginal proposal density for Gaussian random walk MH  and ULA. \emph{Right:} Resulting importance weights. Parameters were chosen to match Gaussian proposal noise in both sampling schemes.}
	\label{fig:Marginal_proposal_density_comic}
\end{figure}

\begin{remark}
\label{remark:SMCIS}
The asymptotic density of the proposed points $Y_k$ is approximated by the MCMC estimate \eqref{equ:HowToEstimateTrda}  which is based on \emph{all} accepted points $X_k$. This results in a runtime quadratic in the number of samples $K$, see Appendix~\ref{sec:runtime}.
Computationally more efficient approximations relying only on subsets of the samples constitute further implementations of MCIS. We will refer to the case where the estimator relies only on a single sample,
\[
\ISs_K(\targfunc)
\coloneqq
\frac{\sum_{k=1}^{K} \widetilde w_k \targfunc(Y_k)}{\sum_{k=1}^{K} \widetilde w_k},
\qquad
\tilde w_k \coloneqq \frac{\trd(Y_k)}{q(Y_k|X_k)},
\]
as Single-MCIS (S-MCIS) and include this in our experimental results for reference.
The S-MCIS estimator is studied in depth in independent work \citep{Rudolf2018}. It satisfies, like MCIS, a law of large numbers and a central limit theorem. Furthermore, \cite{Rudolf2018} show that its asymptotic variance does not depend on the autocorrelation of the Markov chain.
\end{remark}

\begin{remark}
	While $\trd$ is an unnormalized density, $\trd_{Y}$ is, by definition, a probability density function.
	As a consequence, we make the important observation that $\frac{1}{\nsmp}\sum_{k=1}^{\nsmp}  w(Y_k)$ is a consistent estimator for the normalizing constant $ \evid = \int_{\bR^d} \rho(x)\, \mathrm dx$, which in a Bayesian context is also called model evidence or the marginal likelihood of the model.	
	The normalizing constant is an essential quantity for Bayesian model selection, model averaging and testing.
	Our estimator refutes the folk theorem that it is hard to estimate with MCMC generated samples from the original Markov chain alone.
	This might obviate the need for specialized estimators for the normalizing constant or even specialized sampling schemes for model choice such as \cite{Chib1995,Carlin1995}.
\end{remark}

\section{Convergence properties of the MCIS estimator}
\label{section:ConvergenceProperties}
In this section we will take a closer look at the convergence properties of MCIS.
We derive inheritance of geometric and uniform ergodicity for the augmented chain $\left((X_k,Y_k)\right)_{k\in\bN}$ which allows us to derive a law of large numbers (LLN) and a central limit theorem (CLT) for the MCIS estimator $\IS_K$ in Section \ref{section:LLN_CLT_MCIS}.
Its application to MH and ULA is discussed in Section \ref{section:MHandULAfulfillConditions}.
In Section \ref{section:AnalysisISe_nonrigorous} we will give a nonrigorous discussion of why the estimator $\ISe_K$ is meaningful and even outperforms $\IS_K$ in many cases.

\subsection{The augmented chain $(Z_k)_{k\in\bN}$}
The MCIS estimators \eqref{equ:generalIS} and \eqref{equ:generalISestimate} for $\bE_\mu[f]$ are based on the proposals $Y_k,\ k\in\bN$, of the Markov chain acceptance rejection algorithm.
Note, however, that the chain $(Y_k)$ is, in general, not a Markov chain.
Therefore, our analysis will rely on the augmented Markov chain $(Z_k)_{k\in\bN}$ given by
\begin{equation}
\label{equ:AugmentedChain}
Z_k = (X_k,Y_k) \in\bR^{2\targdim}.
\end{equation}
This raises the question which properties of the original chain $(X_k)$ are inherited by the augmented chain $(Z_k)$.
Denoting the transition kernel of  $(X_k)_{k\in\bN}$ by
\[
\mkx\colon \bR^d\times \mathcal B(\bR^d)\to [0,1],
\qquad
\mkx(x,A) = \bP\left[X_{k+1}\in A\, |\, X_k=x\right],
\]
and the one of $(Z_k)_{k\in\bN}$ by $\mkz\colon \bR^{2d}\times \mathcal B(\bR^{2d})\to [0,1]$, we first make the following crucial observation.
Since $X_{k+1}$ is either equal to $X_k$ (with probability $1-\alpha_k$) or to $Y_k$ (with probability $\alpha_k$),
the distribution of $X_{k+1}$ given $Z_k=(X_k,Y_k)$ is a \emph{discrete} distribution with probability mass distributed between the two points $X_k$ and $Y_k$
resulting in a ``degenerate''\footnote{By degenerate we mean here that it is only supported on the two red lines visualized in Figure \ref{fig:AugmentedTransitionKernel}, even if $\prd(\cdot|x)$ is globally supported for all $x$.} transition kernel $\mkz$:
\[
\mkz((x,y),A\times B) = (1-\alpha(x,y))\, \mathds{1}_{A}(x)\,  \prd(B|x) + \alpha(x,y)\, \mathds{1}_{A}(y)\,  \prd(B|y),
\]
where $\mathds{1}_A$ denotes the indicator function of a set $A$ and $\prd(A|x) \coloneqq \int_A \prd(x'|x)\, \mathrm dx'$ is a slight abuse of notation.
\begin{figure}[H]
	\centering
	\includegraphics[width=0.5\textwidth]{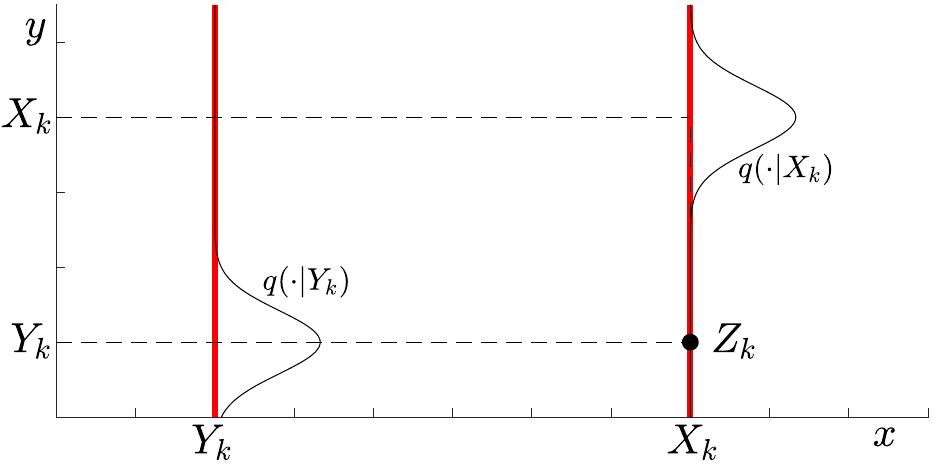}
	\caption{Since $X_{k+1}$ is either equal to $X_k$ or to $Y_k$, the transition kernel $\mkz((X_k,Y_k),\Cdot)$ is degenerate meaning that its support is given by the two red lines.}
	\label{fig:AugmentedTransitionKernel}
\end{figure}
This observation implies that the equality $\mkz((x,y),\Cdot) = \mkz((x',y'),\Cdot)$ can only hold if [$x=x'$ and $y=y'$] or [$x=y'$ and $y=x'$], excluding the existence of accessible atoms of $(Z_k)$.
Similarly, reversibility is never satisfiable for the chain $(Z_k)$.\\
However, while $\mkz^m$ is degenerate for $m=1$, it is globally supported for $m\ge 2$.
Therefore, while unable to retain the properties of the (one-step) transition kernel, the augmented chain still inherits many of those properties of $(X_k)$ which rely on the asymptotic behavior of the kernel as stated in the following theorem.

\begin{restatable}{theorem}{TheoremInheritedProperties}
	\label{theorem:InheritedProperties}
	Let the processes $(X_k)_{k\in\bN},\ (Y_k)_{k\in\bN}$ be given by a Markov chain acceptance rejection algorithm (Algorithm \ref{alg:AcceptReject}).
	Then the augmented chain $(Z_k)_{k\in\bN} = (X_k,Y_k)_{k\in\bN}$ given by \eqref{equ:AugmentedChain} has the following properties.
	\begin{compactenum}[(i)]
		\item \label{theorem:InheritedProperties:stationary}
		If $(X_k)$ has a stationary distribution $\mu_X$ with density $\trd_X$, then $(Z_k)$ has the stationary distribution $\mu_Z$ with density $\trd_Z(x,y) = \trd_X(x)\, \prd(y|x)$.
		\item \label{theorem:InheritedProperties:aperiodicity}
		Let the proposal densities $\prd(\Cdot|\Cdot)$ be globally supported and continuous in both arguments.
		If $(X_k)$ is aperiodic, irreducible and/or Harris positive, so is $(Z_k)$.
		\item \label{theorem:InheritedProperties:geom ergodicity}
		If $(X_k)$ is geometrically ergodic, so is $(Z_k)$.
		\item \label{theorem:InheritedProperties:uniform ergodicity}
		If $(X_k)$ is uniformly ergodic, so is $(Z_k)$.
	\end{compactenum}
\end{restatable}
The proof of the theorem is given in Appendix~\ref{sec:proofs}.

\subsection{Law of large numbers and central limit theorem }
\label{section:LLN_CLT_MCIS}
Since the analysis of the ``practical'' MCIS estimator \eqref{equ:generalISestimate} appears very intricate due to the interdependence between each pair of the points $Y_k$ via the density $\trdae$, we will restrict the analysis to the estimator \eqref{equ:generalIS}, which is typically not implementable in practice.
A discussion of the convergence properties of estimator \eqref{equ:generalISestimate} and why it can outperform \eqref{equ:generalIS} is given in Section \ref{section:AnalysisISe_nonrigorous}.

\begin{restatable}[LLN for MCIS]{theorem}{TheoremLLN}
	\label{theorem:LLN}	
	Let $\trm$ be a probability measure on $\bR^{d}$, $\trd\colon \bR^{d} \to \bR$ be proportional to the probability density function of $\trm$ and $\targfunc\in L^1(\trm)$.
	Let the processes $(X_k)_{k\in\bN},\ (Y_k)_{k\in\bN}$ be given by a Markov chain acceptance rejection algorithm (Algorithm \ref{alg:AcceptReject}) with globally supported proposal densities $\prd(\Cdot|\Cdot)$ that are continuous in both arguments.
	Let the Markov chain $(X_k)$ be aperiodic, irreducible, Harris positive and fulfill Assumption \ref{assumption:MCadmitsUniqueStationaryDensity}.
	Then the MCIS estimator $\IS_K$ given by \eqref{equ:generalIS} fulfills the law of large numbers (LLN):
	\[
	\IS_K(\targfunc)
	\xrightarrow{\rm a.s.}
	\bE_\trm[\targfunc].
	\]
\end{restatable}
The proof of the theorem is given in Appendix~\ref{sec:proofs}.

\begin{restatable}[CLT for MCIS]{theorem}{TheoremCLT}
	\label{theorem:CLT}
	Let the assumptions of Theorem \ref{theorem:LLN} hold, $(Z_k)_{k\in\bN} = ((X_k,Y_k))_{k\in\bN}$ be the augmented chain \eqref{equ:AugmentedChain} and
	\[
	h\colon \bR^{2d}\to\bR^2,
	\qquad
	h(z) = h(x,y) = (\phi(y),w(y))^\top,
	\qquad
	\phi \coloneqq  \frac{\targfunc\trd}{\trda},
	\quad
	w \coloneqq  \frac{\trd}{\trda},
	\]
	where $\trda$ is given by \eqref{equ:HowToEstimateTrda}.
	Further, let the following assumptions hold:
	\begin{enumerate}[(i)]
		\item $(X_k)$ is geometrically ergodic.
		\item $\bE_{Y\sim\trda} \left[ |\phi(Y)|^{2+\epsilon}\right]<\infty$, $\bE_{Y\sim\trda} \left[ |w(Y)|^{2+\epsilon}\right]<\infty$ for some $\epsilon>0$.
	\end{enumerate}
	Then $\|\gamma_h\|<\infty$,  where
	\[
	\gamma_{h} \coloneqq  \frac{\gamma_{h}^{(1)}}{2} + \sum_{k=2}^\infty \gamma_{h}^{(k)},
	\qquad
	\gamma_{h}^{(k)}\coloneqq  \Cov_{X_1\sim\trd} \left[h(Z_1),h(Z_{k})\right] + \Cov_{X_1\sim\trd} \left[h(Z_k),h(Z_{1})\right],
	\]	
	and the MCIS estimator $\IS_K$ given by \eqref{equ:generalIS} fulfills the following central limit theorem (CLT):
	\[
	\sqrt{K}\left(
	\IS_K(\targfunc) - \bE_\trm[\targfunc]
	\right)
	\xrightarrow{\rm d}
	\cN\left(0,\gamma_{\rm CLT}\right),
	\qquad
	\gamma_{\rm CLT}
	\coloneqq  \evid^{-2}
	{\begin{bsmallmatrix} 1 \\ -\bE_\trm[\targfunc] \end{bsmallmatrix}}^{\top} \gamma_h \, \begin{bsmallmatrix} 1 \\ -\bE_\trm[\targfunc] \end{bsmallmatrix}.
	\]	
\end{restatable}
The proof of the theorem is given in Appendix~\ref{sec:proofs}.


\subsection{Application to MH and ULA}
\label{section:MHandULAfulfillConditions}
In order to apply Theorem \ref{theorem:CLT} to specific algorithms such as MH and ULA, we need to verify its assumptions.
Note that the assumption on $(X_k)$ to be aperiodic, irreducible and Harris positive is rather weak.
For MH, if the proposal densities $q(\Cdot | x)$ are globally supported for all $x\in\bR^d$, aperiodicity and irreducibility are straightforward while Harris positivity follows from Lemma 7.3 of \citet{christian2007monte}.
For ULA, the proposal densities are globally supported by construction, which implies aperiodicity and irreducibility.
Harris positivity can be obtained under rather mild conditions similar to the ones ensuring existence of and convergence to the stationary density $\trd_X$, see \citet[Theorem 13.0.1]{meyn2012markov}.
Geometric ergodicity is a stronger assumption and is typically derived by the use of drift criteria \citep[Theorem 15.0.1]{meyn2012markov}.
It has been established under certain conditions for both MH (see \citealt{mengersen1996rates}, \citealt{roberts1996geometric},  Chapter 7 of \citealt{christian2007monte}) and ULA  \citep[see][]{roberts1996exponential,roberts2002langevin,durmus2017nonasymptotic}.
Note that for the analysis of the convergence properties of ULA we do not actually need to consider the augmented chain, since $(Y_k)_{k\in\bN} = (X_{k+1})_{k\in\bN}$ by construction.

\subsection{Convergence of $\ISe_K$ and comparison to LAIS}
\label{section:AnalysisISe_nonrigorous}
In Section~\ref{section:LLN_CLT_MCIS} we analyzed the convergence properties of the MCIS estimator $\IS_K$, since the practically relevant estimator $\ISe_K$ appeared too intricate due to the interdependence of the points $Y_k$.
In this section, we compare the approximation properties of $\ISe_K$ to those of stratified sampling and of layered adaptive importance sampling (LAIS; \citealt{Martino2015}, see Section~\ref{section:RelatedWork}).
We hope that this gives some insight into the way that $\ISe_K$ operates and paves the way to a rigorous mathematical analysis of $\ISe_K$.
However, our attempts to establish such a theory have failed until now.

So far, $\ISe_K$ has only been considered as an approximation to $\IS_K$.
In many experiments however, $\ISe_K$ performs considerably better than $\IS_K$ and very similar to LAIS, while the target density is evaluated only half as often as for LAIS.
This suggests to analyze the estimation properties of $\ISe_K(\targfunc) \approx \bE_\trm[\targfunc]$ directly rather than viewing $\ISe_K$ as an approximation to $\IS_K$.
For this purpose, it is meaningful to compare the samples $Y_k$ to random variables
\[
Z_k \sim \prd(\Cdot|X_k),
\qquad
k=1,\dots,K,
\]
that are sampled independently \emph{after} the Markov chain $(X_k)$ has been established (i.e.\ $Z_1,\dots,Z_K$ are conditionally independent given $X_1,\dots,X_K$).
If we consider $X_1,\dots,X_K$ fixed, $Z_1,\dots,Z_K$ are thereby \emph{stratified} samples (\citeauthor{rubinstein2016simulation}, \citeyear{rubinstein2016simulation}, Chapter 5.5; \citeauthor{mcbook}, \citeyear{mcbook}, Chapter 8.4) from the mixture distribution
\[
\trdae = K^{-1} \sum_{k=1}^{K} \prd(\Cdot|X_k).
\]
This has three advantages:
\begin{itemize}
	\item The samples are drawn from $\trdae$ and not $\trda$, allowing a straightforward application of importance sampling. No approximation of the importance density is necessary.
	\item Stratified samples are provably better than independent samples from $\trdae$ in terms of the variance of the corresponding Monte Carlo estimator \citep[Proposition~5.5.1]{rubinstein2016simulation}.
	\item The samples are actually drawn directly from $\trdae$ without any asymptotic arguments or Markov chains involved.
	The only reason to use the underlying Markov chain $(X_k)$ at all (rather than some arbitrary points $X_k$) is to obtain a ``good'' importance sampling density $\trdae$.
\end{itemize}
Note that $\ISe_K$ based on the samples $Z_k$ in place of $Y_k$ is precisely (a special case of) the LAIS estimator.
Further, the samples $Y_k$ behave very similar to $Z_k$. Each $Y_k$ is a sample from $\prd(\Cdot|X_k)$, only the conditional independence given $(X_k)_{k=1,\dots,K}$ is violated.
This lack of conditional independence is what makes the estimator $\ISe_K$ so difficult to analyze.
However, the conditional dependence of $Y_k$ and $Y_\ell$ can be expected to decline as $|k-\ell|$ grows, since the autocorrelation of the Markov chain is decreasing under certain assumptions, presenting a promising direction for the analysis of $\ISe_K$ that is supported by numerical experiments (compare $\ISe_K$ to LAIS in Figures~\ref{fig:clt}, \ref{fig:MHIS_MoG_GP}, \ref{fig:MCIS_increasing_scale_experiment} and \ref{fig:MHIS_MoG_GP additional}

\section{Related work}
\label{section:RelatedWork}
Previous work derives the density of rejected proposals $\trda^r$ and accepted proposals $\trda^a$ and analyzes the variance reduction from a Rao--Blackwellization perspective.
Note that $\trda$ is related to these densities through the mixture representation $\trda = \acc\trda^a +(1-\acc) \trda^r$, where $\acc$ is the overall acceptance rate.
Using an estimate of $\trda^r$, \cite{casella1996rao} experimentally demonstrate that a subsequent importance sampling estimator for the integral of interest has strongly decreased variance. 
The same paper also proposes the S-MCIS estimator under the assumption that the normalization constant of $\trd$ is known \citep[][ Section 5, estimator $\hat{\tau}_9$]{casella1996rao}.
Alternatively, \cite{douc2011vanilla} use an estimate of $\trda^a$ which reduces the computational cost as compared to the $\trda^r$ based on the approach of \cite{casella1996rao}.
Again, the resulting variance reduction is demonstrated experimentally.
Also, \cite{douc2011vanilla} note the existance of test-function independent control variates. We explore a similar approach in Appendix~\ref{section:ControlVariates}.
\\
The shadow Hamiltonian Monte Carlo (SHMC) approach \citep{izaguirre2004shadow}, little known in machine learning and statistics, is a special case of MCIS.
In particular, it relies on the observation that HMC without an MH correction does not sample from the desired Hamiltonian density $\trd^{h}$ but rather from a modified density $\trda^{h}$.
Again, the idea is to approximate $\trda^{h}$ and use the MCIS estimator.
Like MCIS in general, SHMC is known to strongly outperform standard HMC.\\
Gradient importance sampling \citep[GRIS,][]{schuster2015gradient} is closely related to S-MCIS.
It uses an unadjusted Langevin sampling scheme together with an adaptation of the covariance matrix used for Gaussian proposals.
The weighting for the estimator only uses a single time step.
The main difference is that GRIS introduces a novel type of proposal adaptation scheme for the well-known population Monte Carlo sampling-importance-resampling scheme and estimator (PMC, \cite{cappe2004population}), while MCIS represents a new type of estimator for unmodified Markov chain algorithms.\
Markov adaptive multiple importance sampling \cite[MAMIS,][]{martino2015mcmc} is a sampling scheme related to PMC. It uses a set of samples (called particles), but uses MCMC evolution of the particles unlike PMC, which employs a sampling-importance-resampling scheme.
\\
Adaptive multiple importance sampling \cite[AMIS,][]{cornuet2012adaptive,marin2012consistency} is an algorithm and estimator in the population Monte Carlo family which uses a weighting of samples with a mixture very similar to the MCIS estimator $\trdae$.
The approach taken in AMIS is that of Rao--Blackwellization of a PMC estimator.
However, the AMIS consistency proof is tailored to PMC algorithms   \citep{marin2012consistency} and does not apply to our Markov chain case.
\\
Layered adaptive importance sampling \cite[LAIS,][]{Martino2015} suggests to first use the MH algorithm as a way of producing samples approximately from the target.
In a second step, LAIS builds a mixture distribution based on the MH samples and draws a new set of samples which are corrected with importance sampling for final integral estimation.
The component distributions of the mixture do not need to be the same as the proposal distributions used for MH.
While the LAIS estimator has the same form as the MCIS estimator, its two-stage approach necessitates new evaluations of the target density in the importance sampling stage, whereas MCIS recycles the proposed points $Y_k$ instead of drawing new samples and thereby avoids additional target evaluations. The computation time saved this way comes at the cost of an additional interdependence of the points $Y_k$ as discussed in Section \ref{section:AnalysisISe_nonrigorous}, which makes the convergence properties of the MCIS estimator $\ISe_K$ more difficult to analyze.
\\
The method suggested by \cite{yuan2013novel} uses the same estimator as LAIS and MCIS without deriving its correctness. The proposed sampling method is a modified Metropolis--Hastings scheme, where several samples are generated from the proposal distribution at each step of MH and are candidates for the next state of the Markov Chain. However, possible advantages of the proposed scheme are not discussed.
\\
\cite{botev2013markov} discuss an approach also termed Markov Chain importance sampling, which however is markedly different from ours.
They first generate samples following the \emph{optimal} importance sampling density $\trd^\ast \propto \trd \, |f|$ using MCMC and fit a density model $\widehat{\trd^\ast}$ to those. New samples generated from $\widehat{\trd^\ast}$ using ordinary Monte Carlo are then used for standard IS. Our approach uses unaltered MCMC samples with a novel IS estimator.
\\
Finally, the waste recycling scheme for MCMC also aims to use rejected samples to derive a more efficient estimatior \citep{frenkel2004speed,delmas2009does}.
However, this is accomplished by analytically marginalizing out the probability of a proposal being accepted, i.e.\ taking the weighted average of proposal and current point, with weights given by the acceptance probability.
It was shown that this can be seen as a control variate approach \citep{delmas2009does}, and that an optimal control variate coefficient is available.
MCIS, on the other hand, derives the marginal distribution of the proposals, which is a fundamentally different approach.

\section{Numerical Experiments}
\label{section:Numerical}

This section contains the main experiments. We discuss the effects of proposal scaling of the different estimators  in Appendix~\ref{sec:Effects of proposal scaling}. Runtime of the MCIS and MCMC estimators is compared in Appendix~\ref{sec:runtime}.\\
We considered two artificial target distributions and one arising from the Bayesian posterior of a Gaussian process (GP) regression model. The samplers were tuned to yield acceptance rate reported in the literature as optimal for the vanilla MCMC estimator.
Further performance gains can be expected from tuning the acceptance rate for MCIS.

For each reported experiment, we ran the MCMC algorithm for $10\,000$ iterations and repeated the sampling  procedure $20$ times with different random seeds. The test function was given by
$f(x) = \targdim^{-1}\sum_{i=1}^{\targdim} \varphi(x_i)$ where $\varphi(z) = z^3$ and $x_i$ is the $i$th of $d$ overall dimensions in $x$. For demonstrating the CLT, we plot logarithmic mean standard deviation vs. CPU time. For demonstrating convergence, we plot mean logarithmic absolute error vs. CPU time. Means are taken with a sliding window after pooling the measurements from all repeated runs.
We obtained similar results for further test functions, namely for $\varphi(z) = z$, $\varphi(z) = z^2$ and $\varphi(z) = \exp(z)$, and provide the plots for those in the appendix.
The reported CPU time is the sum of sampling and estimate computation times.

The central limit theorem for MCIS using Markov chains induced by both Langevin and Metropolis--Hastings is demonstrated in Figure~\ref{fig:clt}, which plots the standard deviation of the estimators as a function of computation time.

\begin{figure}[htb]
	\centering
	\includegraphics[width=0.49\textwidth]{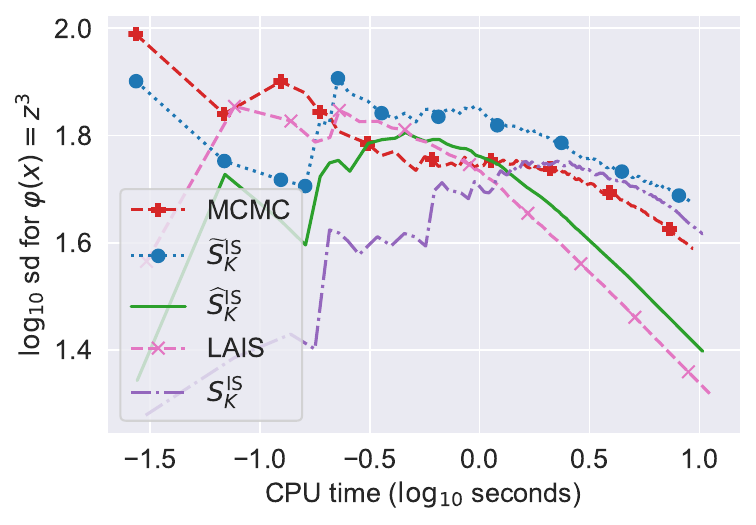}
	\includegraphics[width=0.49\textwidth]{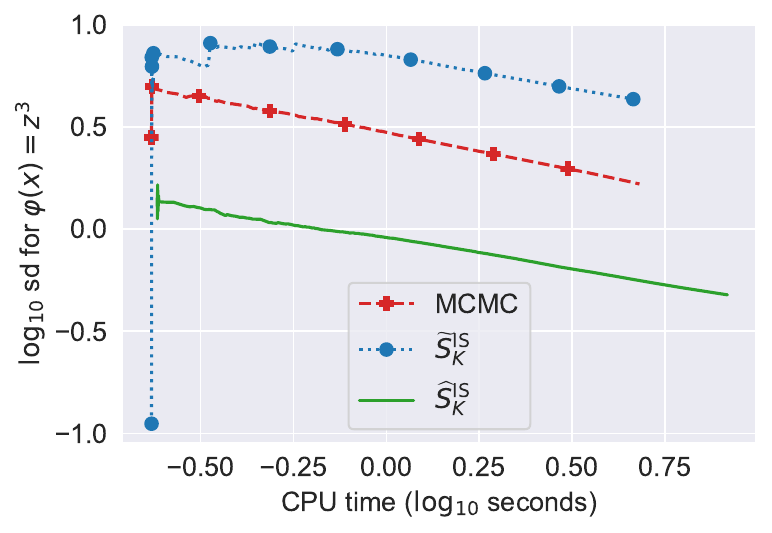}
	\caption{ Standard deviation over CPU time for estimating $\mathbb{E}[f]$, $f(x)= \targdim^{-1}\sum_{i=1}^{\targdim} x_i^3$. \emph{Left:} Gaussian random walk MH, mixture of Gaussians target. \emph{Right:} Langevin algorithm with Gaussian target.}
	\label{fig:clt}
\end{figure}

What follows is a description of all target distributions we used in the experimental section.

\paragraph{Gaussian target.} This unimodal Gaussian distribution is isotropic where each dimension follows the law $\mathcal{N}(5, 0.7^2)$. We expected this to be an easy target for both vanilla and all MCIS estimators. We work in $d=3$ dimensions, unless stated otherwise.

\paragraph{Mixture of Gaussians target.} We consider the following mixture of two Gaussians in $d=3$ dimensions:
\[
\trm = \tfrac{1}{2}\left (\mathcal{N}(3\cdot\mathds 1, 0.7^2 I_\targdim) + \mathcal{N}(7\cdot \mathds 1, 1.5^2 I_\targdim)\right ),
\]
where $\mathds 1 \coloneqq  (1,1,1)^\intercal$.
We expected this to be an easy target for estimators based on random walk MCMC samples, but problematic when using ULA samples because of the multimodality.

\paragraph{GP regression with automatic relevance determination.} In this real-life problem, Gaussian process regression is applied to noise level prediction. 
We used data aquired by NASA through testing airfoil blade sections in an anechoic wind tunnel\footnote{available at \href{https://archive.ics.uci.edu/ml/datasets/airfoil+self-noise}{https://archive.ics.uci.edu/ml/datasets/airfoil+self-noise}}. The dataset provided six predictors for the noise level. As some of these where on a log-scale, we exponentiated them and centered the data in a preprocessing step.
A Gaussian automatic relevance determination (ARD) kernel was used, inducing individual variances for all  predictors as model parameters. Another parameter was the likelihood scale of the GP, resulting in a seven dimensional posterior.
We used independent standard normal distributions as priors on $\log(\exp(\sigma^2_*) - 1), \log(\lambda)$, where $\sigma^2_*$ are ARD variances and $\lambda$ is the likelihood scale.
Ground truth for the posterior was obtained using $100\,000$ iterations from MH with optimal scaling and the standard estimator. This was repeated three times, resulting in $300\,000$ overall samples for establishing ground truth.

In what follows, we discuss the convergence for MCIS based on unadjusted Langevin and MH chains using the different targets in detail. 

\subsection{Unadjusted Langevin MCIS}
We tested MCIS with ULA samples for both the Gaussian and the mixture of Gaussians target.
We expected the multimodality of the latter to pose a particular problem for ULA, which, however, is caused by the sampling algorithm rather than the estimator used. The time discretization parameter was set to $0.1$.
\begin{figure}[htb]
	\centering
	\trimmedplot{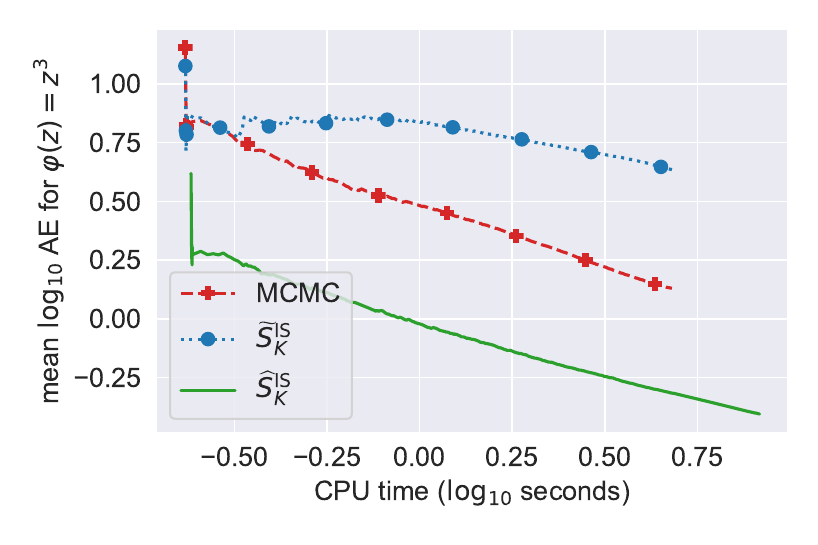}
	\trimmedplot{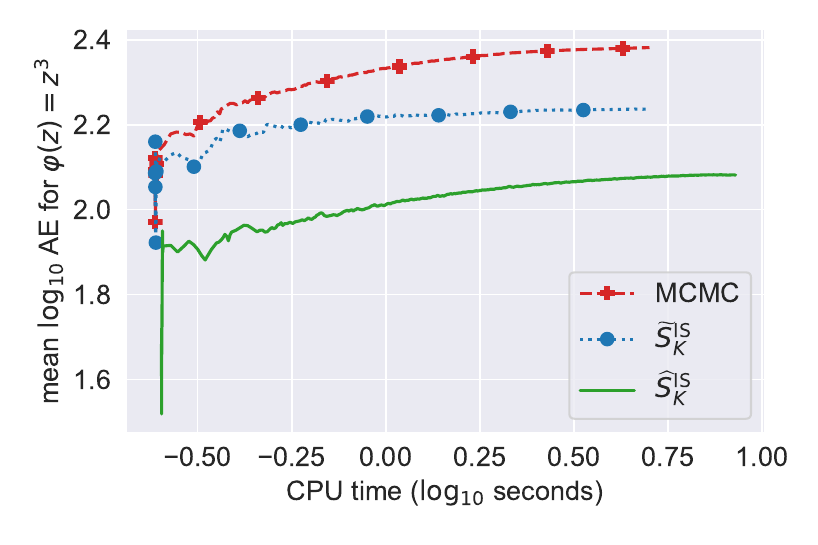}
	\caption{Unadjusted Langevin algorithm. Absolute error over CPU time for estimating $\mathbb{E}[f]$, $f(x)= \targdim^{-1}\sum_{i=1}^{\targdim} x_i^3$, with the standard estimator and the proposed estimators. \emph{Left:} Gaussian target. \emph{Right:} Mixture of Gaussians target.}
	\label{fig:DLIS}
\end{figure}
\paragraph{Gaussian target.}
As expected, the Gaussian target posed no problem whatsoever to Langevin sampling with either the vanilla or MCIS estimators.
However, S-MCIS ($\ISs_K$) failed to provide sufficient accuracy to improve over the standard estimator. 
On the other hand, MCIS improved estimation accuracy considerably.
The mean logarithmic absolute error of $\bE[f]$ is plotted over CPU time in Figure~\ref{fig:DLIS} (left).
We can observe that S-MCIS and the vanilla estimator take approximately the same amount of time for sampling and estimation combined, while MCIS takes some extra time for the same amount of samples.
However, if we keep the amount of CPU time fixed as opposed to the number of samples, MCIS still outperforms the vanilla estimator.

\paragraph{Mixture of Gaussians target.}
As expected, we observe problems of ULA when sampling a multimodal target, as all estimators have not yet entered the stage of convergence (Figure~\ref{fig:DLIS}, right).
This demonstrates that ULA is not a good choice for sampling multimodal target distributions. Consequently, recent papers on nonasymptotic analysis of ULA schemes make unimodality assumptions \citep{durmus2017nonasymptotic,dalalyan2017theoretical,cheng2017underdamped}.
However, for this difficult target the MCIS estimators very clearly improve upon the state of the art.

\subsection{Random walk MCIS}
In these experiments we used a Gaussian random walk proposal inside Metropolis--Hastings sampling and compare the standard estimator with variants of the MCIS estimator. The sampler was tuned to an acceptance rate of around $0.234$ which has been reported to be optimal for the vanilla MCMC estimator in the literature \citep{roberts2001optimal}.

\begin{figure}[htb]
	\centering
	\trimmedplot{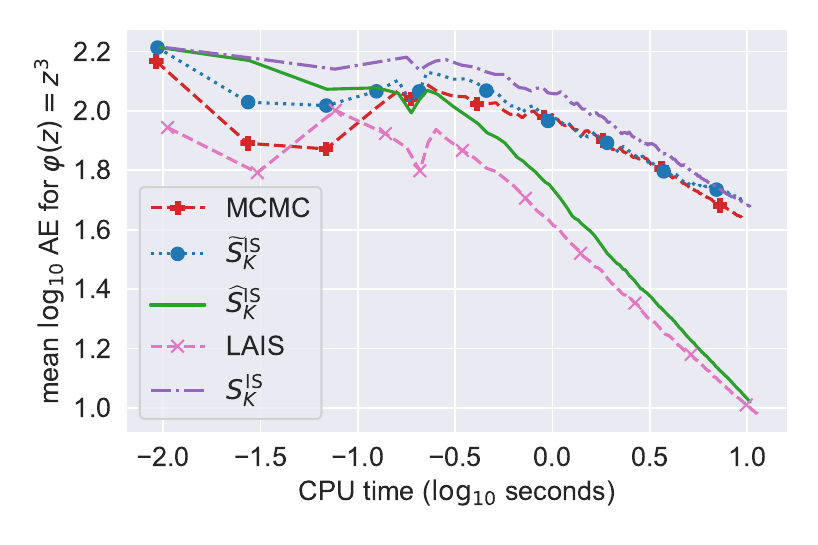}
	\trimmedplot{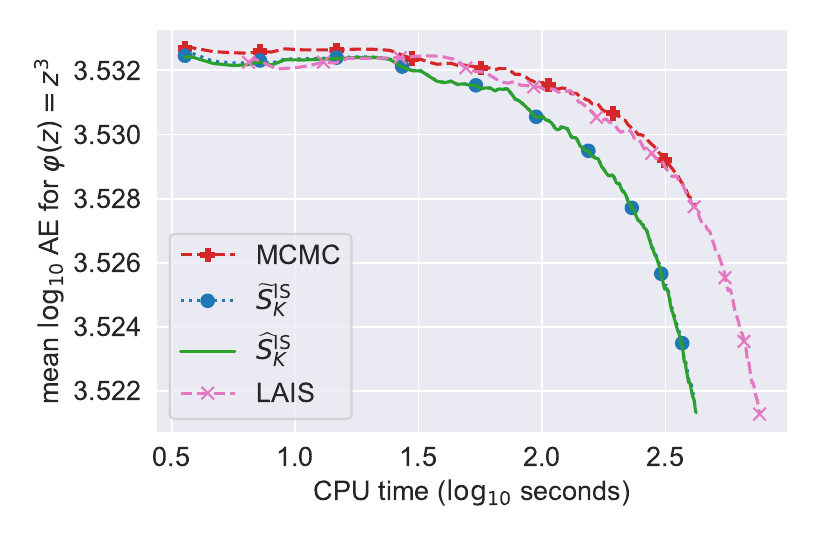}
	\caption{ Gaussian random walk MH. Absolute error over CPU time for estimating $\mathbb{E}[f]$, $f(x)= \targdim^{-1}\sum_{i=1}^{\targdim} x_i^3$, with the standard estimator and the proposed estimators. \emph{Left:} Mixture of Gaussians target. \emph{Right:} GP target.}
	\label{fig:MHIS_MoG_GP}
\end{figure}

\paragraph{Mixture of Gaussians target.}
Gaussian random walk proposals used a noise scale of $\param = 1.8$. The efficiency gains of using full MCIS are huge (Figure~\ref{fig:MHIS_MoG_GP} left). Owing to the fact that the artificial target distribution is very cheap to evaluate, LAIS and MCIS perform comparably well.

\paragraph{GP regression target.}
For this seven dimensional target, we used a second order Taylor approximation in order to estimate the posterior covariance $\Sigma$ around a MAP point.
Preliminary runs where used to find a scaling of $\Sigma$ that resulted in near-optimal acceptance rate.
Again, there are efficiency gains when using the MCIS estimator.
Since this target distribution is computationally more expensive, MCIS outperforms LAIS, as the latter requires twice as many target evaluations.
In fact, experimentally, MCIS achieves the same precision as LAIS in almost exactly half the computation time.

\section{Conclusion and outlook}
\label{section:Conclusion}
In this paper, we propose a novel type of Monte Carlo estimators when using Markov chain algorithms for sampling.
Our approach is based on the elementary observation that the proposed samples $Y_k$ of most Markov chain acceptance rejection methods (Algorithm \ref{alg:AcceptReject}) have an asymptotic density $\trda$ that is related to, but different from the target density $\trd$, and that importance sampling can be used to correct for this change of measure, resulting in our first MCIS estimator $\IS_K$ given by \eqref{equ:generalIS}.
Since the density $\trda$ is typically not available in closed form, we propose to estimate it by the Monte Carlo sum \eqref{equ:HowToEstimateTrda} and obtain our second, practically applicable MCIS estimator $\ISe_K$ given by \eqref{equ:generalISestimate}.
\\
The advantages of using MCIS are twofold.
First, it makes use of \emph{all} proposed samples $Y_k$ rather than just the accepted ones.
Since the (possibly costly) target density needs to be evaluated in each $Y_k$, MCIS thereby avoids a waste of resources and often yields a higher precision at nearly the same computational cost.
Second, the asymptotic density $\trd_X$ of the Markov chain $(X_k)$ does not need to coincide with the target density $\trd$ for MCIS to be consistent.
The application of MCIS to ULA, which samples from a different distribution than the target $\trm$, is a particular realization of this idea.
Importantly, as a by-product, we also obtain an estimator for the normalizing constant, which is an essential quantity for Bayesian model selection, model averaging and testing.
\\
We prove a law of large numbers and a central limit theorem for the MCIS estimator $\IS_K$, while the convergence properties of $\ISe_K$ are discussed less rigorously and compared to LAIS in Section \ref{section:AnalysisISe_nonrigorous}, where we also argue that it is insufficient to view $\ISe_K$ just as an approximation to $\IS_K$ and that it can outperform the latter in many cases.
\\
Usually, using the MCIS estimator does not result in a large increase in computational time, since no further target density evaluations are necessary.
We study its runtime in Appendix \ref{sec:runtime} and propose alternative realizations of the MCIS idea that have a linear runtime in Remark \ref{remark:SMCIS}, in particular the S-MCIS estimator studied by \citet{Rudolf2018}.
We have examined the full MCIS and S-MCIS estimators empirically and found full MCIS to be more CPU-efficient than the standard Monte Carlo estimator, sometimes to a large extent.
\\
Several questions are left for future work.
For one, we did not study optimal tuning of the sampler for the purpose of using MCIS subsequently.
Instead, we used the optimal acceptance rate literature for tuning guidelines \citep{roberts2001optimal}, which are tailored to the standard estimator.
Further performance gains can be expected from such a tuning.
Secondly, we believe that the problem of infinite variance that can occur in importance sampling is of no issue for MCIS rather under mild assumptions.
This intuition comes from the fact that, in the case of Gaussian random walk MH, $\trda$ will be the convolution of $\trd$ with a Gaussian kernel and thereby have heavier tails.
This insight needs to be formalized and generalized to other Markov chain acceptance rejection algorithms such as ULA and pCN \citep{cotter2013mcmc,rudolf2018generalization}.
Thirdly, an interesting direction would be to derive error bounds for finite and fixed numbers of samples as has previously been done for certain classes of target distributions and the unadjusted Langevin sampler with the vanilla estimator  \citep{durmus2017nonasymptotic,dalalyan2017theoretical,cheng2017underdamped}.
Finally, a rigorous analysis of the approximation properties of $\ISe_K$ remains an open problem, cf.\ the discussion in Section \ref{section:AnalysisISe_nonrigorous}.

\textbf{Paper contains $4600$ words in total.}
\paragraph{Acknowledgements} \unblindinfo{We thank Alain Durmus, Bj\"orn Sprungk, Daniel Rudolf and Péter Koltai for discussions about different aspects of this work.
Marcus Weber answered questions about shadow HMC.
Christian Robert pointed out related Rao--Blackwellization literature.
The comments of the anonymous reviewers strengthened this paper considerably.\\
Ingmar Schuster was financed by a PSL Research University postdoc grant and Deutsche Forschungsgemeinschaft under SFB 1114, projects A06 and B03, performing the work while at Paris Dauphine and FU Berlin.
Ilja Klebanov was financed by the Einstein Foundation Berlin, ECMath Project CH13.}{Acknowledgements}

\bibliographystyle{abbrvnat}
\bibliography{Bibliography}

\newpage

\appendix

\section{Proofs}
\label{sec:proofs}

\begin{proof}[Proof of Theorem \ref{theorem:InheritedProperties}]
	Property \eqref{theorem:InheritedProperties:stationary} follows directly from the definition of $(Z_k)$.
	Aperiodicity and irreducibility of $(Z_k)$ in \eqref{theorem:InheritedProperties:aperiodicity} follow directly since the proposal densities $\prd$ are assumed to be globally supported.
	For Harris recurrence, let $A\subseteq\mathcal B(\bR^{2d})$ with $\mu_Z(A)>0$.
	By continuity of $\prd$ there exist subsets $A_X,A_Y\in\mathcal B(\bR^d)$ and $\epsilon>0$, such that
	\begin{equation}
	\label{equ:RectangleConstruction}
	A_X\times A_Y\subseteq A,
	\qquad
	\mu_X(A_X)>0
	\qquad\text{and}\qquad
	\int_{A_Y} \prd(y|x)\, \mathrm dy > \epsilon\text{ for all }x\in A_X.
	\end{equation}
	Since $(X_k)$ is Harris positive by assumption, we have
	\[
	\bP\left[X_k\in A_X \text{ infinitely often}\right] = 1,
	\]
	regardless of the initial value $X_1$.
	Using \eqref{equ:RectangleConstruction} this implies that, almost surely, $Z_k$ will enter $A\supseteq A_X\times A_Y$ in finite time,
	\[
	\bP\left[ \min_{k>1}(Z_k\in A) < \infty \right] = 1,
	\]
	regardless of the initial value $X_1$. By \cite[Proposition 6.33]{christian2007monte} this proves \eqref{theorem:InheritedProperties:aperiodicity}.
	If $(X_k)$ is geometrically ergodic, there exists $r>1$ such that
	\[
	\sum_{m=1}^{\infty} r^m\, \|\mkx^m(x,\Cdot)-\mu\|_{\rm TV} < \infty
	\]
	for all $x\in\bR^d$.
	For any signed measure $\nu$ on $\bR^d$ we define the signed measure $\nu\odot \prd$ on $\bR^{2d}$ by
	\[
	\nu\odot \prd (A\times B) \coloneqq  \int_A\int_B \prd(y|x)\, \mathrm dy\, \mathrm d\nu(x).
	\]
	Since $\prd$ is non-negative, its Jordan decomposition $\nu\odot \prd = (\nu\odot \prd)_+ - (\nu\odot \prd)_-$ is given by
	\[
	(\nu\odot \prd)_+ = \nu_+\odot \prd,
	\qquad
	(\nu\odot \prd)_- = \nu_-\odot \prd,
	\]
	and therefore
	\[
	\|\nu\odot \prd\|_{\rm TV} = (\nu\odot \prd)_+(\bR^{2d}) + (\nu\odot \prd)_-(\bR^{2d}) = \nu_+(\bR^d) + \nu_-(\bR^d) = \|\nu\|_{\rm TV}.
	\]
	Here, we used the connection between the total variation norm and the Jordan decomposition of a measure, see e.g. \cite[Chapter 2.4]{sullivan2015introduction}.
	Since $X_2$ is either equal to $X_1$ (with probability $1-\alpha_1$) or to $Y_1$ (with probability $\alpha_1$), this implies for each $(x,y)\in\bR^{2d}$, $\alpha\coloneqq \alpha(x,y)$ and $m \ge 1$
	\begin{align*}
	\mkz^m((x,y),\Cdot)
	&=
	\left[(1-\alpha)\mkx^{m-1}(x,\Cdot) + \alpha \mkx^{m-1}(y,\Cdot)\right]\odot \prd,
	\\[0.1cm]
	\|\mkz^m((x,y),\Cdot)-\mu\odot \prd\|_{\rm TV}
	&\le
	(1-\alpha)\|\mkx^{m-1}(x,\Cdot)-\mu\|_{\rm TV} + \alpha \|\mkx^{m-1}(y,\Cdot)-\mu\|_{\rm TV}
	\end{align*}
	and thus
	\[
	\sum_{m=1}^{\infty} r^m\, \|\mkz^m((x,y),\Cdot)-\mu\odot \prd\|_{\rm TV}
	<
	\infty,
	\]
	which proves property \eqref{theorem:InheritedProperties:geom ergodicity}.
	The proof of \eqref{theorem:InheritedProperties:uniform ergodicity} goes analogously.
\end{proof}

Our proof of the CLT (Theorem \ref{theorem:CLT}) relies on a well-known version of the CLT for Markov chains, namely \citet[Theorem 2]{chan1994discussion} (see also \citealt[Corollary 2]{jones2004on}, or \citealt[Theorem 6.67]{christian2007monte}). However, we will first have to generalize this result to \emph{multivariate} quantities of interest $h$, which is the statement of the following proposition:

\begin{proposition}
\label{prop:multivariateCLT}
Let $(Z_k)_{k\in\bN}$ be  an aperiodic, irreducible, Harris positive and geometrically ergodic Markov chain in $\bR^n$ with stationary distribution $\pi$ and $h\colon\bR^n\to\bR^m$ be a measurable function such that $\bE_{Z\sim\pi} \left[ \|h(Z)\|^{2+\epsilon}\right]<\infty$ for some $\epsilon>0$.
Then $\|\gamma_h\|<\infty$, where $\gamma_h\in\bR^{m\times m}$ is given by
\[
\gamma_{h} \coloneqq  \frac{\gamma_{h}^{(1)}}{2} + \sum_{k=2}^\infty \gamma_{h}^{(k)},
\qquad
\gamma_{h}^{(k)}\coloneqq  \Cov_{Z_1\sim\pi} \left[h(Z_1),h(Z_{k})\right] + \Cov_{Z_1\sim\pi} \left[h(Z_k),h(Z_{1})\right],
\]
and
\[
G_K := \sqrt{K}\left( \frac{1}{K}\sum_{k=1}^{K} h(Z_k) - \bE_{Z\sim\pi}[h(Z)] \right)
\xrightarrow[K\to\infty]{\rm d}
\cN(0,\gamma_h).
\]
\end{proposition}
\begin{proof}
Let $v\in\bR^m$ and $h_{v} = v^{\top}h\colon \bR^{n}\to\bR$.
The above conditions imply $\bE_{Z\sim\pi} \left[ |h_{v}(Z)|^{2+\epsilon}\right]<\infty$, hence \citet[Theorem 2]{chan1994discussion} yields
$0 \le \gamma_{h_{v}} < \infty$ and
\begin{equation}
\label{equ:ProjectedConvergenceInDistribution}
v^{\top} G_K
=
\sqrt{K}\left( \frac{1}{K}\sum_{k=1}^{K} h_{v}(Z_k) - \bE_{Z\sim\pi}[h_{v}(Z)] \right)
\xrightarrow[K\to\infty]{\rm d}
\cN(0,\gamma_{h_{v}}).
\end{equation}
Note that $\gamma_{h_{v}}^{(k)} = v^\top \gamma_{h}^{(k)} v$ and thereby $\gamma_{h_{v}} = v^\top \gamma_{h} v$.
Since $v\in\bR^{m}$ was chosen arbitrary and $\gamma_{h}$ is symmetric (for each partial sum), this implies $\|\gamma_{h}\| < \infty$.
Since $G_K$ converges in distribution to $H\sim \cN(0,\gamma_h)$ if and only if
\[
v^{\top} G_K\xrightarrow[K\to\infty]{\rm d}v^{\top}H \sim \cN(0,v^{\top}\gamma_h v),
\]
for each $v\in\bR^m$, \eqref{equ:ProjectedConvergenceInDistribution} proves the claim.
\end{proof}

\begin{proof}[Proof of Theorems \ref{theorem:LLN} and \ref{theorem:CLT}]
	The ergodic theorem (see e.g.\ \citealt[Theorem 6.63]{christian2007monte}) yields for the numerator and denominator of \eqref{equ:generalIS}:
	\begin{align*}
	\overline\phi_K
	\coloneqq \ 
	&\frac{1}{K}\sum_{k=1}^{K}\phi(Y_k)
	\xrightarrow{\rm a.s.}
	\bE_{Y\sim \trda}\left[\frac{\targfunc\trd}{\trda}(Y)\right]
	=
	\int \targfunc(y)\, \trd(y)\, \mathrm dy
	=
	\evid\, \bE_\trm[\targfunc],
	\\
	\overline w_K
	\coloneqq \ 		
	&\frac{1}{K}\sum_{k=1}^{K}w (Y_k)
	\xrightarrow{\rm a.s.}
	\bE_{Y\sim \trda}\left[\frac{\trd}{\trda}(Y)\right]
	=
	\int \rho(y)\, \mathrm dy
	=
	\evid.
	\end{align*}
	This already proves the LLN for $\IS_K(\targfunc) = \overline \phi_K/\overline w_K$ (Theorem \ref{theorem:LLN}).
	Since $(Z_k)$ is aperiodic, irreducible, Harris positive and geometrically ergodic by Theorem~\ref{theorem:InheritedProperties}, Proposition~\ref{prop:multivariateCLT} yields
	\[
	\sqrt{K}\left(
	\begin{pmatrix} \overline \phi_K \\ \overline w_K \end{pmatrix}
	-
	\begin{pmatrix} \evid\bE_\trm[\targfunc] \\ \evid \end{pmatrix}
	\right)	
	\xrightarrow{\rm d}
	\cN(0,\gamma_h).
	\]
	By applying the delta method \cite[Theorem 11.2.14]{lehmann2006testing} with
	\[
	g(u,v) = \frac{u}{v},
	\qquad
	\nabla g \big(\evid\bE_\trm[\targfunc],\evid\big)
	=
	\evid^{-1}
	\begin{pmatrix}
	1 \\ -\bE_\trm[\targfunc]
	\end{pmatrix}
	\]
	we obtain the stated CLT for $\IS_K(\targfunc) = \overline \phi_K/\overline w_K$.
\end{proof}

\section{Control Variates}
\label{section:ControlVariates}

We can add a control variate to the random variable $\prd(y|X_k)$ for all $k$ in order to reduce the variance of the Monte-Carlo estimate
\[
\trdae(y) = \frac{1}{K}\sum_{k=1}^{K} \prd(y|X_k) \approx \trda(y).
\]
The advantage in the particular combination with our MCIS estimator is that using a control variate to improve the estimate of $\trda$ will improve the final integral estimation  \emph{independent} of the particular test function. This is a strong difference to the standard application of control variates, which necessitates additional computation for each new integrand.
To this end, we replace $\prd(y|X)$ by 
\begin{equation}
\label{equ:ControlVariateGeneral}
\tilde \prd(y|X,X') = \prd(y|X) - c\, \big(\prd(y|X) - \prd(y|X')\big),
\end{equation}
where $X'$ is the successor of $X$ in the Metropolis Hastings algorithm, $c\in\bR$ is a constant and the control variate has mean zero, $\bE [\prd(y|X) - \prd(y|X')]=0$ for each $y\in\bR^d$, as $X$ and $X'$ have the same distribution (under the assumption that $X \sim \trm$).

In what follows, let $y\in\bR^d$ be fixed. It is well-known \cite[Chapter 4.1]{glasserman2013monte} that the optimal coefficient $c^* = c^*(y)$ is given by
\[
c^*
=
\frac{\Cov \big[\prd(y|X), \prd(y|X) - \prd(y|X')\big]}{\bV \big[\prd(y|X) - \prd(y|X')\big]}
=
\frac{\bE\big[\prd(y|X)^2 - \prd(y|X)\prd(y|X')\big]}{\bE\big[(\prd(y|X) - \prd(y|X'))^2\big]},
\]
which can easily be approximated by another Monte Carlo estimate:
\[
c^*
\approx
\frac{\sum_{k=1}^{K-1} \prd(y|X_k)^2 - \prd(y|X_k)\prd(y|X_{k+1})}{\sum_{k=1}^{K-1} (\prd(y|X_k)-\prd(y|X_{k+1}))^2}.
\]
Note that for each evaluation point $y\in\bR^d$ (in our case $y = Y_1,\dots,Y_K$), $c^* = c^*(y)$ has to be computed \emph{separately}.
Nevertheless, the new estimator requires only little additional computational effort. This is because $\prd(y|X_k)$ has to be evaluated anyway for each $k=1,\dots,K$ and $y = Y_1,\dots,Y_K$ for the overall MCIS estimator $\ISe_K$ and no new target density or test function evaluations are necessary.

The new estimator \eqref{equ:ControlVariateGeneral} no longer guarantees non-negativity of the density estimate $\trdae$. This can be avoided by adding a control variate to the random variable $\ell(y|X)=\log \prd(y|X)$ instead of $\prd(y|X)$,
\begin{equation*}
\label{equ:LogControlVariateGeneral}
\tilde \ell(y|X,X') = \ell(y|X) - c\, \big(\ell(y|X) - \ell(y|X')\big).
\end{equation*}
The additional computation effort is negligible, as $\log \prd(y|X)$ has to be computed for the full MCIS estimator anyway.
Of course, other control variates are possible and more than one control variate can further reduce variance.
However, it is unclear under what circumstances control variates actually reduce variance and leave in-depth analysis of this question to future work.

\section{Code}

\lstset{ 
	backgroundcolor=\color{white},   
	basicstyle=\footnotesize,}
This appendix contains code for the CPU-intensive computation of $\trdae(Y_i)$ for all $Y_i$ in the Python programming language.
\begin{lstlisting}[language=Python]
import numpy as np
from numpy import log
from scipy.special import logsumexp

def location_mixture_logpdf(samps, prop_means, prop_weights,
                            zeroprop, contr_var = False):
    """
    CPU-efficient (i.e. fully vectorized) way to compute the
    logpdf of a weighted location mixture
    
    Parameters
    ==========
    
    samps        -   samples to compute the logpdf of, one per row
    prop_means   -   proposal means from Metropolis-Hastings,
                     one per row (locations of the location mixture)
    prop_weights -   number of times that proposals where drawn
                     with this proposal mean, this has the same size
                     as prop_means.shape[0]
                     (i.e. number of times that the current state 
                      is repeated in Metropolis-Hastings)
    zeroprop     -   proposal distribution centered at zero,
                     used only for logpdf evaluations
    contr_var    -   whether or not to use control variates
    """
    diff = samps - prop_means[:, np.newaxis, :]
    
    tmp_shape = [np.prod(diff.shape[:2]), diff.shape[-1]]
    #compute raw logpdfs
    lpdfs = zeroprop.logpdf(diff.reshape(tmp_shape))
    #get them back into shape
    lpdfs = lpdfs.reshape(diff.shape[:2])
    #compute normalization weights
    logprop_weights = log(prop_weights/prop_weights.sum())[:, np.newaxis]
    if not contr_var: 
        return logsumexp(lpdfs + logprop_weights, 0)
    else:
        time0 = lpdfs + logprop_weights + log(len(prop_weights))
        time1 = np.hstack([time0[:,1:],time0[:,:1]])
        cov = np.mean(time0**2 - time0*time1)
        var = np.mean((time0-time1)**2)
        lpdfs = lpdfs - cov/var * (time0 - time1)        
        return logsumexp(lpdfs - log(len(prop_weights)), 0)
\end{lstlisting}

\section{Effects of proposal scaling}
\label{sec:Effects of proposal scaling}
We used the unimodal Gaussian target in $10$ and $30$ dimensions to test for the effects of different scalings of a Gaussian random walk proposal in MH. We started with a very peaked Gaussian random walk proposal which resulted in an acceptance probability close to $1$. Subsequently, we used a heuristic to decrease the acceptance rate down to $0$ by incrementally scaling up the bandwidth of the Gaussian proposal, so that the space between acceptance rates $0$ and $1$ was covered. At each proposal scaling, we ran $10$ independent MH samplers for $5000$ iterations. Absolute errors for an integral estimate were computed for each run, ordered by acceptance rate of the run, and smoothed by a rolling average. For the result, see Figure~\ref{fig:MCIS_increasing_scale_experiment}. \\
This reveals that there is indeed a scaling regime in which $\ISe_K$ outperforms the standard MCMC estimator for the $10$ dimensional target. Interestingly, $\IS_K$ always performs slightly better than the standard estimator. Even when increasing the dimension of the Gaussian target distribution, $\IS_K$ often improved over the standard estimator. This was not true for $\ISe_K$ however, which performs worse for the unimodal Gaussian in $30$ dimensions. For the mixture of Gaussians target in high dimensions we could still observe an advantage when using $\ISe_K$.

\begin{figure}
	\centering
	\trimmedplot{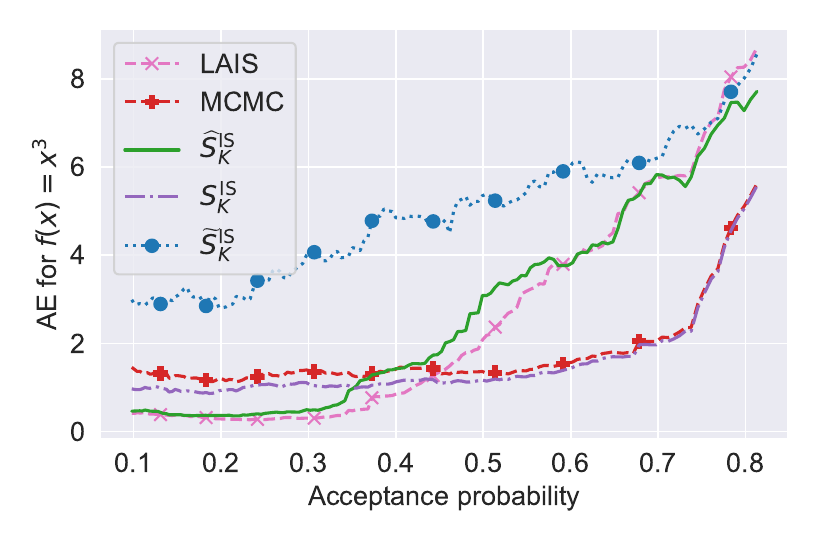}
	\trimmedplot{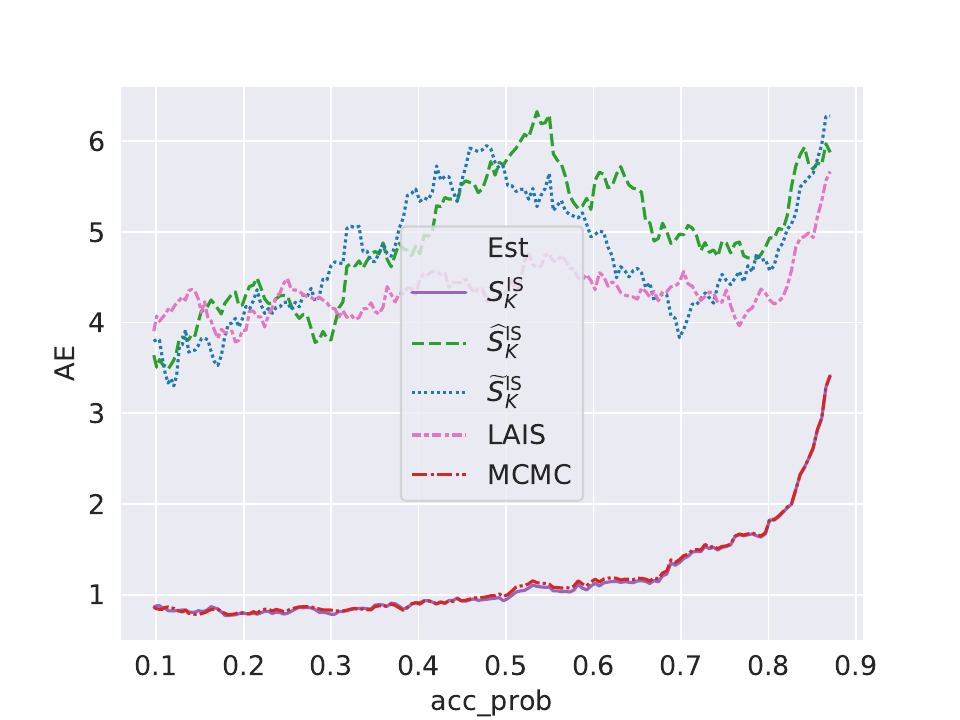}	
	\caption{Effects of proposal scaling. \emph{Left}: unimodal Gaussian in $10$d. \emph{Right:} unimodal Gaussian in $30$d.}
	\label{fig:MCIS_increasing_scale_experiment}
\end{figure}

\section{Runtime considerations}
\label{sec:runtime}
Let $\alpha \in [0,1]$ be the overall acceptance rate and $c_\targfunc$, $c_\prdunparam$, $c_\trd$ denote the computational cost of evaluating the test function $\targfunc$, the proposal density $\prd$ and the target density $\trd$, respectively, at a single point.
Then, using MCIS instead of the vanilla MH estimator increases the computational cost by a factor of
\begin{equation}
\label{eq:prolongation}
1+\frac{(1-\alpha )\nsmp c_\targfunc+\alpha\nsmp^2 c_\prdunparam}{\alpha \nsmp c_\targfunc+\nsmp c_\trd + 2\nsmp c_\prdunparam}\, .
\end{equation}
In other words, when the evaluation of the target density is computationally costly and $c_\trd$ is large, the relative amount of extra time for using MCIS instead of the vanilla estimator becomes small, and indeed goes to zero as $c_\trd$ increases as long as $\nsmp, \alpha, c_\targfunc$ and $c_\prdunparam$ remain fixed. For example, in Bayesian inference an increasing dataset size implies an increase of $c_\trd$, as evaluating the target involves evaluating the likelihood for each data point.

Using all samples for estimating $\trda$ makes the estimator $\ISe_K$ quadratic in computational complexity.
If the quadratic runtime is of concern, one possible remedy is S-MCIS or a middle ground of using some fixed number $j$ of samples for estimating $\trda$ with $1<j<K$, retaining the linear runtime (in $K$) of the vanilla MH estimator.
However, we obtained a very CPU-efficient estimator by choosing a completely vectorized approach to implement full MCIS.\\
To get an intuition for the prolongation factor, observe that convergence plots for MCIS with synthetic targets (see e.g.\ Figure~\ref{fig:DLIS}) end at higher CPU time values, which reflects the fact that the prolongation factor \eqref{eq:prolongation} is strictly larger than one.
The good performance of full MCIS  in this case is especially remarkable since evaluating the target was very cheap ($c_\trd$ was small).
The GP target on the other hand arises from an actual Bayesian posterior and the cost of evaluating the target is considerably higher than for the synthetic target experiments ($c_\trd$ is large compared to $c_\prdunparam$).
This is noticeable in the plots, as the CPU time for all estimators is almost the same for the same number of samples.
In other words, the prolongation factor is close to one.
Of course, the runtime for full MCIS increases dramatically if the code is not vectorized.

\section{Additional test functions}
\label{sec:Additional plots}
In this section we provide convergence plots for additional test functions.
The functions used were $f(x) = \targdim^{-1}\sum_{i=1}^{\targdim} \varphi(x_i)$ with $\varphi(z) = z$, $\varphi(z) = z^2$ and $\varphi(z) = \exp(z)$.

\begin{figure}[ht]
	\centering
	\trimmedplot{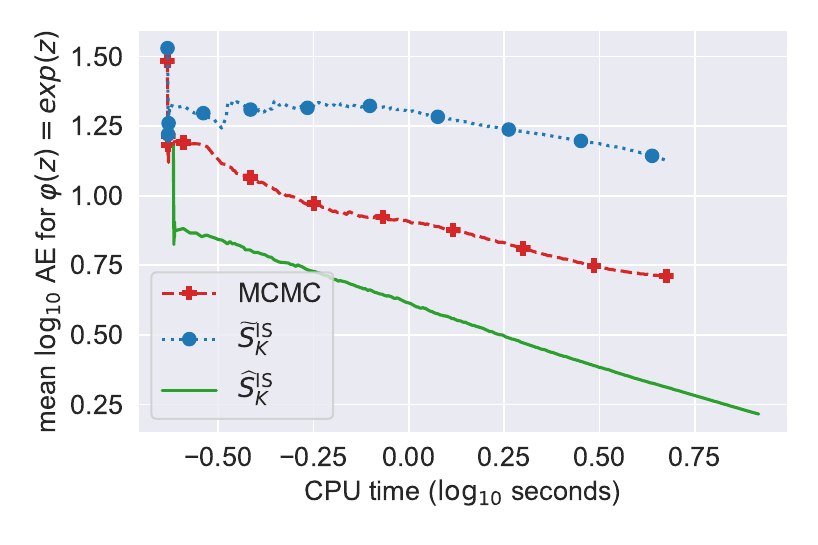}
	\trimmedplot{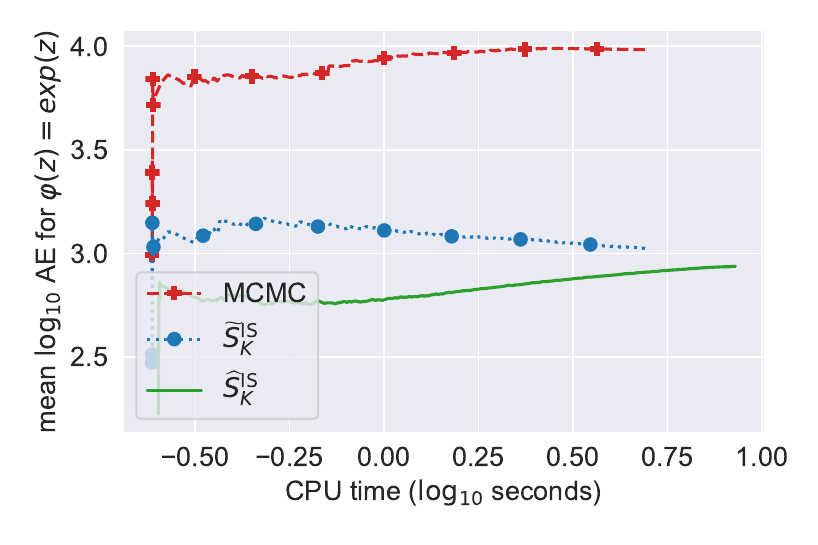}
	\trimmedplot{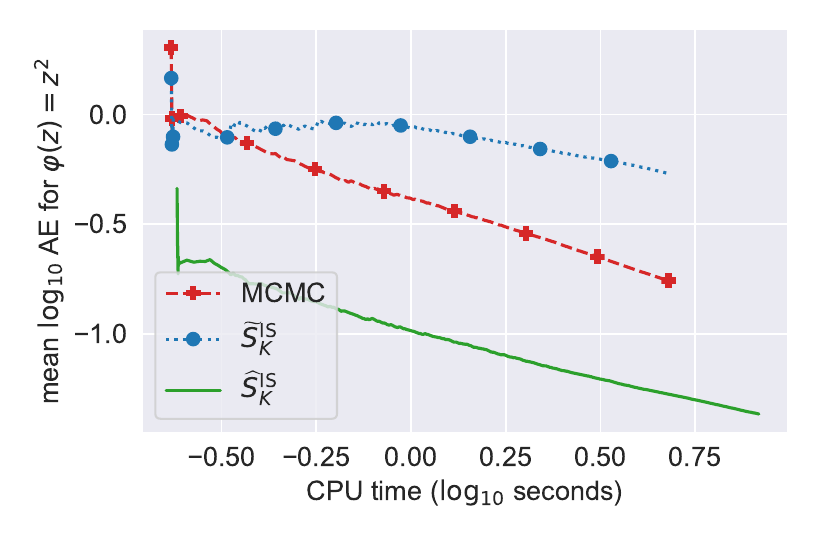}
	\trimmedplot{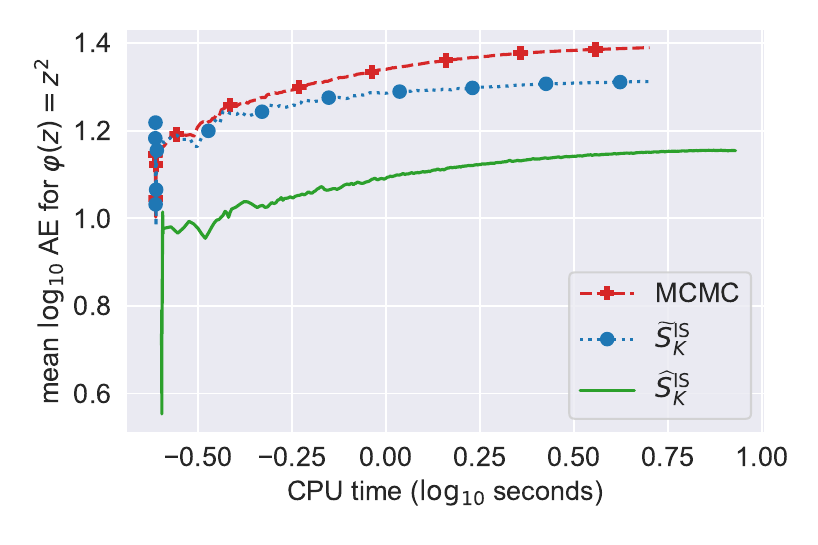}
	\trimmedplot{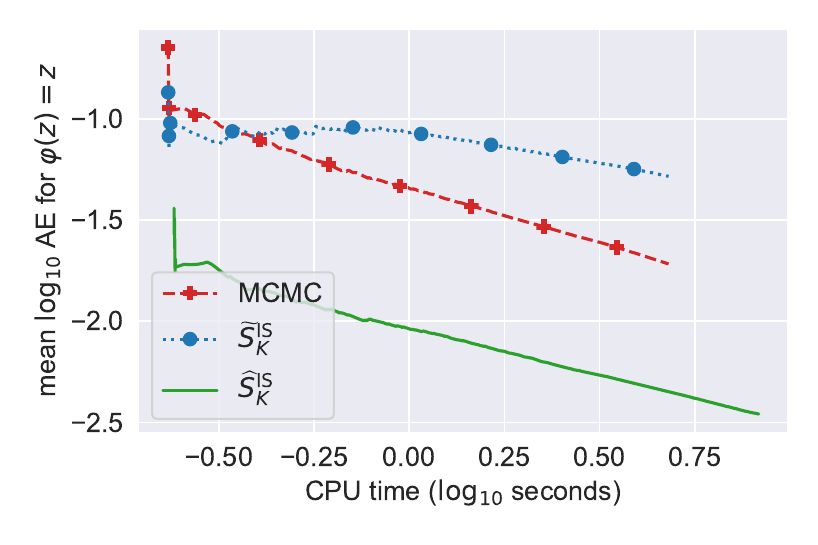}
	\trimmedplot{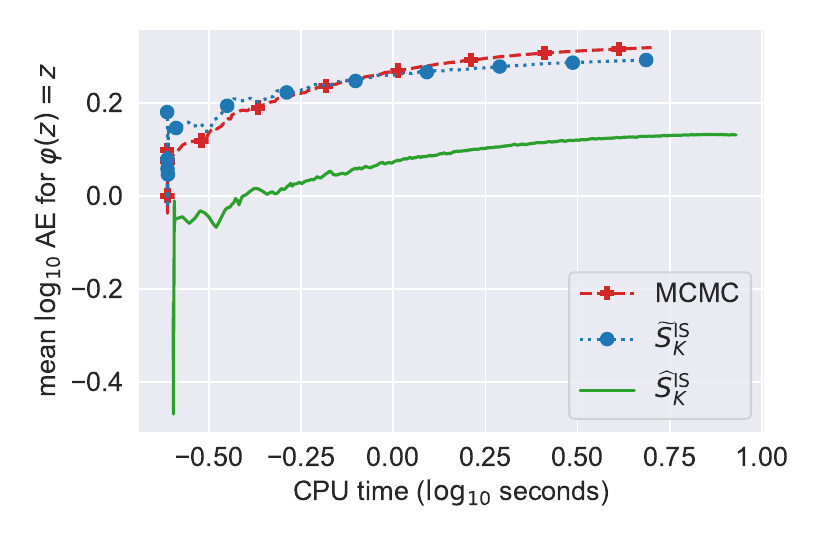}
	\caption{Unadjusted Langevin. Absolute error over CPU time for estimating expectations of several test functions with the standard estimator and the proposed estimators. Errors averaged across dimensions and $20$ independent MCMC repetitions. \emph{Left:} Gaussian target. \emph{Right:} Mixture of Gaussians target.}
	\label{fig:DLIS additional}
\end{figure}

\begin{figure}[ht]
	\centering
	\trimmedplot{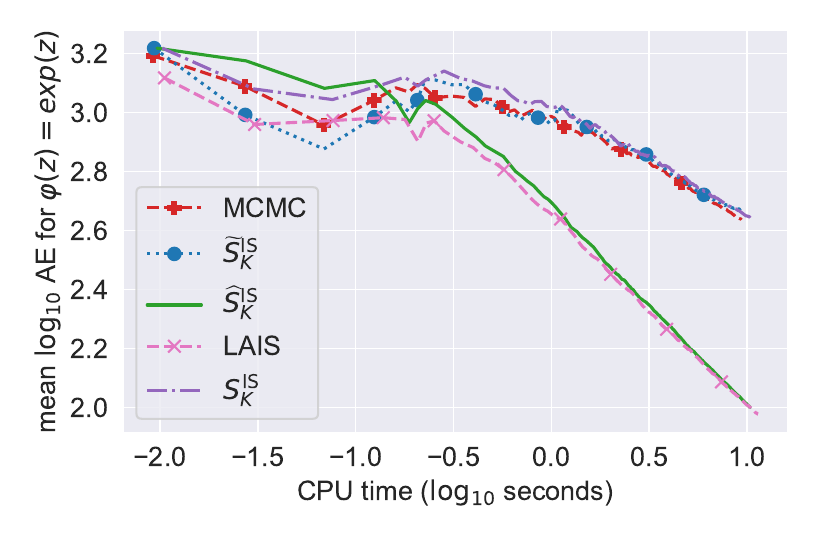}
	\trimmedplot{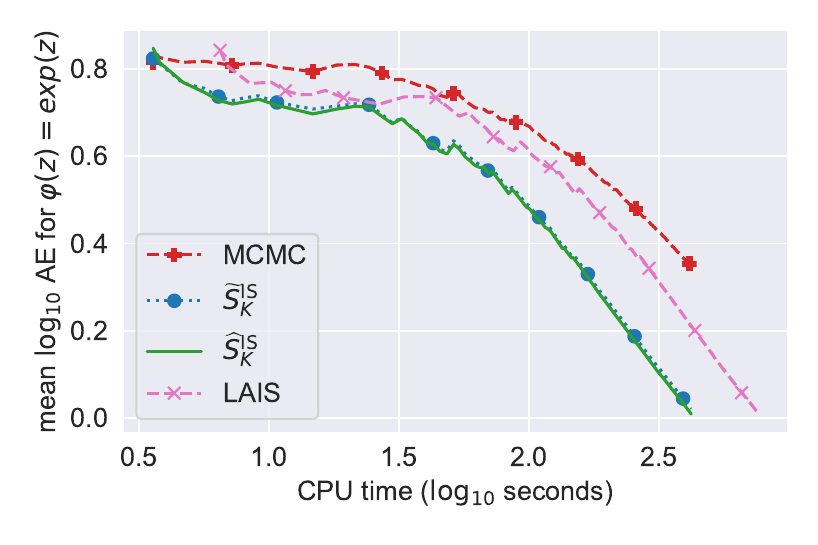}
	\trimmedplot{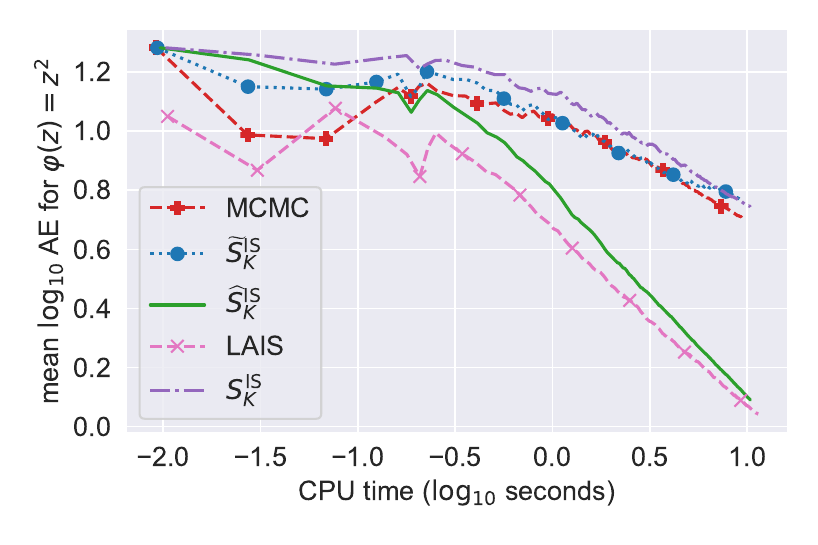}
	\trimmedplot{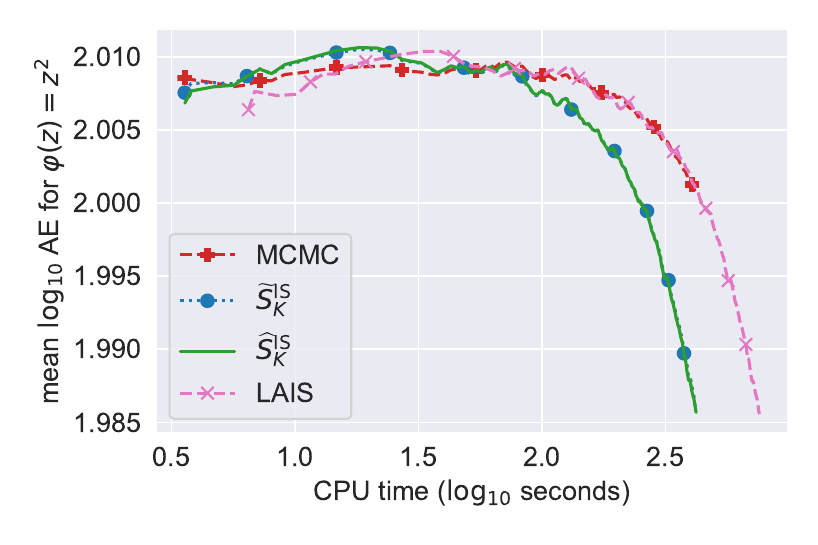}
	\trimmedplot{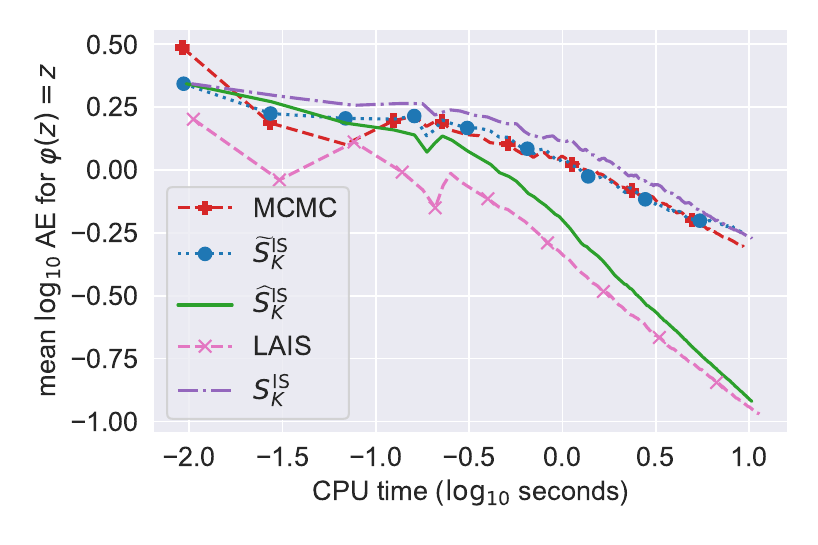}
	\trimmedplot{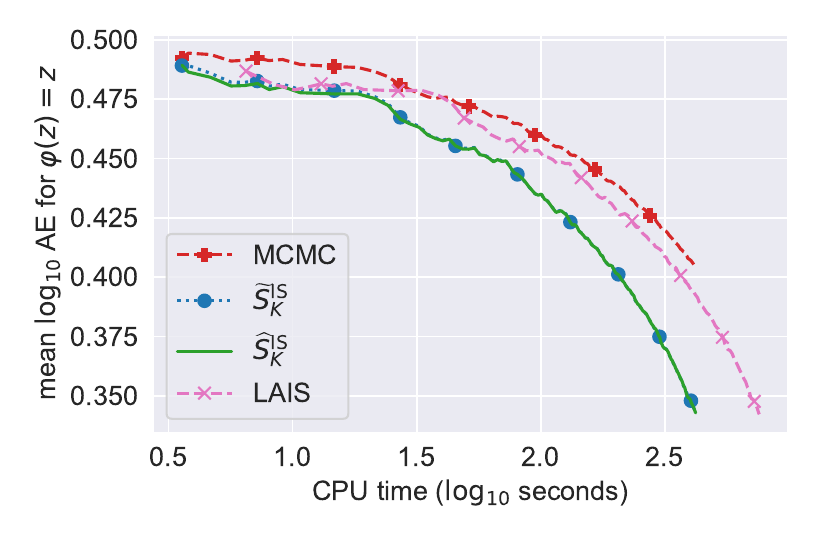}
	\caption{ Gaussian random walk MH. Absolute error over CPU time for estimating expectations of several test functions with the standard estimator and the proposed estimators. Errors averaged across dimensions and $20$ independent MCMC repetitions. \emph{Left:} Mixture of Gaussians target. \emph{Right:} GP target.}
	\label{fig:MHIS_MoG_GP additional}
\end{figure}

\end{document}